\pdfoutput=1

\documentclass{article} 
\usepackage{nips15submit_e,times}
\usepackage{hyperref}
\usepackage{url}

\usepackage{epstopdf}
\usepackage{graphicx}
\usepackage{epsfig}
\usepackage{multirow}
\usepackage{amsmath,amsthm,amssymb}
\usepackage[title]{appendix}

\graphicspath{{figures/}} 

\newtheorem{assumption}{Assumption}
\newtheorem{definition}{Definition}
\newtheorem{theorem}{Theorem}

\newcommand{\hh}{h}
\newcommand{\Xone}{X_1} 
\newcommand{\Xtwo}{X_2}
\newcommand{\cXone}{{\cX_1}}
\newcommand{\cXtwo}{{\cX_2}}

\newcommand{\bongram}{bag-of-$n$-gram}

\newcommand{\cX}{{\mathcal X}}
\newcommand{\cY}{{\mathcal Y}}
\newcommand{\cH}{{\mathcal H}}
\newcommand{\bw}{{\mathbf w}}
\newcommand{\bx}{{\mathbf x}}
\newcommand{\bz}{{\mathbf z}}

\newcommand{\bp}{{\mathbf p}}
\newcommand{\bv}{{\mathbf v}}

\newcommand{\psz}{p}

\newcommand{\voc}{V}
\newcommand{\vsz}{|\voc|}
\newcommand{\wrd}{w}

\newcommand{\cnn}{CNN}
\newcommand{\cnns}{CNNs}

\newcommand{\Activ}{\boldsymbol{\sigma}}
\newcommand{\activ}{\sigma}

\newcommand{\wei}{\bw}
\newcommand{\Wei}{{\mathbf W}}
\newcommand{\Bias}{{\mathbf b}}
\newcommand{\Vei}{{\mathbf V}}

\newcommand{\region}{{\mathbf r}} 
\newcommand{\iL}{\ell} 

\newcommand{\myre}{\mathbb{R}}
\newcommand{\rdim}{q} 
\newcommand{\nodes}{m} 

\newcommand{\scnn}{seq-CNN}

\newcommand{\bcnn}{bow-CNN}

\newcommand{\nbw}{NB-LM}

\newcommand{\ourEmb} {tv-embedding}
\newcommand{\ourEmbs} {tv-embeddings}
\newcommand{\OurEmb} {Tv-embedding}

\newcommand{\ourEmbd}{tv-embedded}
\newcommand{\ourEmbShort} {tv-embed.}

\newcommand{\ourEmbUns}  {unsup-tv.}
\newcommand{\ourEmbPar}  {parsup-tv.}
\newcommand{\ourEmbUnsN} {unsup3-tv.}

\newcommand{\bow}{{bow}}

\title{Semi-supervised Convolutional Neural Networks for Text
  Categorization via Region Embedding}

\author{
Rie Johnson \\
RJ Research Consulting \\
Tarrytown, NY, USA \\
\texttt{riejohnson@gmail.com} \\
\And
Tong Zhang\thanks{
  Tong Zhang would like to acknowledge NSF IIS-1250985, NSF IIS-1407939, 
  and NIH R01AI116744 for supporting his research.
}
 \\
Baidu Inc., Beijing, China \\
Rutgers University, Piscataway, NJ, USA \\
\texttt{tzhang@stat.rutgers.edu} \\
}

\nipsfinalcopy 

\newcommand{\tightDisplayBegin}[1]{\begingroup \setlength{\belowdisplayskip}{#1} \setlength{\belowdisplayshortskip}{#1} \setlength{\abovedisplayskip}{#1} \setlength{\abovedisplayshortskip}{#1}}
\newcommand{\tightDisplayEnd}{\endgroup}

\newcommand{\tightArrayBegin}[1]{\begingroup \renewcommand\arraystretch{#1}}
\newcommand{\tightArrayEnd}{\endgroup}

\makeatletter
\newcommand\tightpara{\@startsection{paragraph}{4}{\z@}{0.5ex plus
   0ex minus 0.2ex}{-1em}{\normalsize\bf}}   
\makeatother 

\begin{document}
\maketitle
\begin{abstract}
This paper presents a new semi-supervised framework with convolutional neural networks (CNNs) 
for text categorization. 
Unlike the previous approaches that rely on word embeddings, our method 
learns embeddings of small text regions 
from unlabeled data 
for integration into a supervised \cnn. 
The proposed scheme for embedding learning is based on the idea of
two-view semi-supervised learning, which is
intended to be useful for the task of interest even though the training is done on unlabeled data.  
Our models achieve better results than previous approaches on 
sentiment classification and topic classification tasks.
\end{abstract}

\section{Introduction} 

Convolutional neural networks (\cnns)
\cite{LeCun+etal98}
are neural networks that can make use of the internal structure of data 
such as the {\em 2D structure} of image data through convolution layers, 
where each computation unit responds to a small region of input data 
(e.g., a small square of a large image).  
%
On text, 
\cnn\
has been gaining attention, used in systems for 
tagging, 
entity search,  
sentence modeling, 
and so on 
\cite{nnnlpICML08,nnnlpJMLR11,XS13,Gao+etal14,Shen+etal14,KGB14,XLLZ14,Tang+etal14,WCA14,Kim14}, 
to make use of the {\em 1D structure} (word order) of text data.  
Since \cnn\ was originally developed for image data, which is fixed-sized, low-dimensional and dense,
without modification it cannot be applied to text documents, which are variable-sized, 
high-dimensional and sparse if represented by sequences of one-hot vectors.  
In many of the \cnn\ studies on text, therefore, 
words in sentences are first converted to low-dimensional {\em word vectors}. 
The word vectors are often obtained by some other method 
from an additional large corpus, which is 
typically done in a fashion similar to language modeling though 
there are many variations \cite{BDVJ03,nnnlpICML08,MH08,TRB10,DFU11,wvecNips13}.  

Use of word vectors obtained this way is a form of semi-supervised learning and 
leaves us with the following questions.  
Q1. How effective is \cnn\ on text in a purely supervised setting without the aid of unlabeled data? 
Q2. Can we use unlabeled data with \cnn\ more effectively than
using general word vector learning methods?  
Our recent study \cite{JZ15} addressed Q1 on text categorization
and showed that \cnn\ without a word vector layer is not only feasible 
but also beneficial when not aided by unlabeled data. 
Here we address Q2 also on text categorization: 
building on \cite{JZ15}, we propose a new semi-supervised framework that learns embeddings of 
small text {\em regions} (instead of {\em words}) from unlabeled data, for use in a supervised \cnn.  

The essence of \cnn, as described later, is to convert small regions of data (e.g., ``love it'' in a document) 
to feature vectors for use in the upper layers; in other words, through training, a convolution 
layer learns an {\em embedding}
of small regions of data. 
Here we use the term `embedding' loosely to mean a structure-preserving
function, in particular, a function that generates low-dimensional features that preserve the predictive structure. 
\cite{JZ15} applies \cnn\ {\em directly to high-dimensional one-hot vectors}, 
which leads to {\em directly} learning an {\em embedding}
of {\em small text regions} (e.g., regions of size 3 like phrases, 
or regions of size 20 like sentences), eliminating 
the extra layer for word vector conversion. 
This direct learning of region embedding 
was 
noted 
to have the merit of 
higher accuracy with a simpler system (no need to tune hyper-parameters for word vectors) 
%
than supervised word vector-based \cnn\ in which word vectors are randomly initialized and trained 
as part of \cnn\ training. 
Moreover, 
the performance of \cite{JZ15}'s best \cnn\ rivaled or exceeded the previous best results
 on the benchmark datasets. 

Motivated by this finding, we seek effective use of unlabeled data for text categorization 
through {\em direct learning of embeddings of text regions}. 
Our new semi-supervised framework 
learns a {\em region embedding} from {\em unlabeled data} 
and uses it to produce additional input (additional to one-hot vectors) 
to supervised \cnn, where a {\em region embedding} is trained 
with {\em labeled data}.  
Specifically,  
from unlabeled data, we learn {\em \ourEmbs} (`tv' stands for
`two-view'; defined later) of a text region 
through the task of predicting its surrounding context.
According to our theoretical finding, a {\em \ourEmb} has desirable
properties under ideal conditions on the relations between two views and the labels.  
While in reality the ideal conditions may not be perfectly met, 
we consider them as guidance in designing the tasks for \ourEmb\ learning. 

We consider several types of \ourEmb\ learning task trained on unlabeled data; e.g., one task is 
to predict the presence of the concepts relevant to the intended task (e.g., `desire to recommend the product') 
in the context, and we indirectly use labeled data to set up this task.  
Thus, we seek to learn \ourEmbs\ {\em useful specifically for the task of interest}.  
This is in contrast to the previous word vector/embedding learning methods, which 
typically produce a word embedding for general purposes
so that all aspects (e.g., either syntactic or semantic) of words are captured. 
 In a sense, the goal of our region embedding learning is to map text regions to high-level concepts 
 relevant to the task.  This cannot be done by word embedding learning since individual words in isolation 
 are too primitive to correspond to high-level concepts. 
   For example, ``easy to use'' conveys positive sentiment, but ``use'' in isolation does not.   
%
We show that our models with \ourEmbs\ outperform 
the previous best results on sentiment classification 
and topic classification.  
Moreover, a more direct comparison confirms that our region \ourEmbs\ provide more 
{\em compact and effective}
representations of regions for the task of interest 
than what can be obtained by manipulation of a word embedding. 

\subsection{Preliminary: one-hot \cnn\ for text categorization \cite{JZ15}}
\label{sec:textcnn}

A \cnn\ is a feed-forward network equipped with convolution layers 
interleaved with pooling layers.  
%
A convolution layer consists of computation units, 
each of which responds to a small region of input (e.g., a small square of an image), 
and the small regions 
collectively cover the entire data. 
A computation unit associated with the $\iL$-th region of input $\bx$ 
computes: 
\tightDisplayBegin{3pt}
\begin{align}
\Activ( \Wei \cdot \region_\iL(\bx) + \Bias )~, \label{eq:activ}
\end{align}
\tightDisplayEnd
where $\region_\iL(\bx) \in \myre^{\rdim}$ is the input {\em region vector} that represents 
the $\iL$-th region.    
Weight matrix $\Wei \in \myre^{\nodes \times \rdim}$ and 
bias vector $\Bias \in \myre^{\nodes}$ are {\em shared} by all the units in the same layer, 
and they are learned through training.  
In \cite{JZ15}, input $\bx$ is a document represented by one-hot vectors (Figure \ref{fig:textcnn}); 
therefore, we call \cite{JZ15}'s \cnn\ {\em one-hot \cnn};  
$\region_\iL(\bx)$ 
can be either a concatenation of one-hot vectors, a bag-of-word vector (\bow), or 
a bag-of-$n$-gram vector: 
e.g., for a region ``love it''
\newcommand{\wonv}[1]{\mbox{\small #1}}
\newcommand{\wonvb}[1]{\mbox{\small \bf #1}}
\tightArrayBegin{0.01}
\begin{align}
&\begin{array}{ccccccccccll} 
      &\wonv{I}&\wonv{it}&\wonvb{love}&&\wonv{I}&\wonvb{it}&\wonv{love}& \cr
     {\small \region_\iL(\bx)=}[ &0&0&1& | &0&1&0& ]^\top && \mbox{(concatenation)}\cr  \label{eq:seqconv}
\end{array}\\
&\begin{array}{cccccclllllllll} 
  &\wonv{I}&\wonvb{it}&\wonvb{love} \cr
  {\small \region_\iL(\bx)=}[ &0&1&1&]^\top &&&&&&&&& \mbox{(\bow)}\cr  \label{eq:bowconv}
\end{array} 
\end{align}
\tightArrayEnd
The \bow\ representation (\ref{eq:bowconv}) loses word order within the region but is more robust to data sparsity, 
enables a large region size such as 20, and speeds up training by having fewer parameters.  
This is what we mainly use for embedding learning from unlabeled data.  
\cnn\ with (\ref{eq:seqconv}) is called {\em \scnn} and 
\cnn\ with (\ref{eq:bowconv}) {\em \bcnn}.  
The region size and stride (distance between the region centers) are meta-parameters.
Note that we used a tiny three-word vocabulary for the vector examples above to save space, 
but a vocabulary of typical applications could be much larger. 
$\Activ$ in (\ref{eq:activ}) is a component-wise non-linear function (e.g., applying $\activ(x)=\max(x,0)$ to 
each vector component).  Thus, each computation unit generates an $\nodes$-dimensional vector 
where $\nodes$ is the number of weight vectors ($\Wei$'s rows) or {\em neurons}.  
In other words, {\em a convolution layer embodies an embedding of text regions}, which produces an $\nodes$-dim vector for each 
text region.  
In essence,
a region embedding uses co-presence and absence of words in a region as input 
to produce predictive features, 
e.g., if presence of ``easy to use'' with absence of ``not'' is a predictive indicator, 
it can be turned into a large feature value by having a negative weight on ``not'' 
(to penalize its presence) and positive weights on the other three words
in one row of $\Wei$.  
  A more formal argument can be found in the Appendix.
%
The $\nodes$-dim vectors from all the text regions of each document are aggregated by 
the pooling layer, by either component-wise maximum ({\em max-pooling}) or average ({\em average-pooling}), 
and used by the top layer (a linear classifier) as features for classification.  
Here we focused on the convolution layer; 
for other details, \cite{JZ15} should be consulted.

\begin{figure}
\centering
\begin{minipage}[b]{0.385\linewidth}
\includegraphics[width=\linewidth]{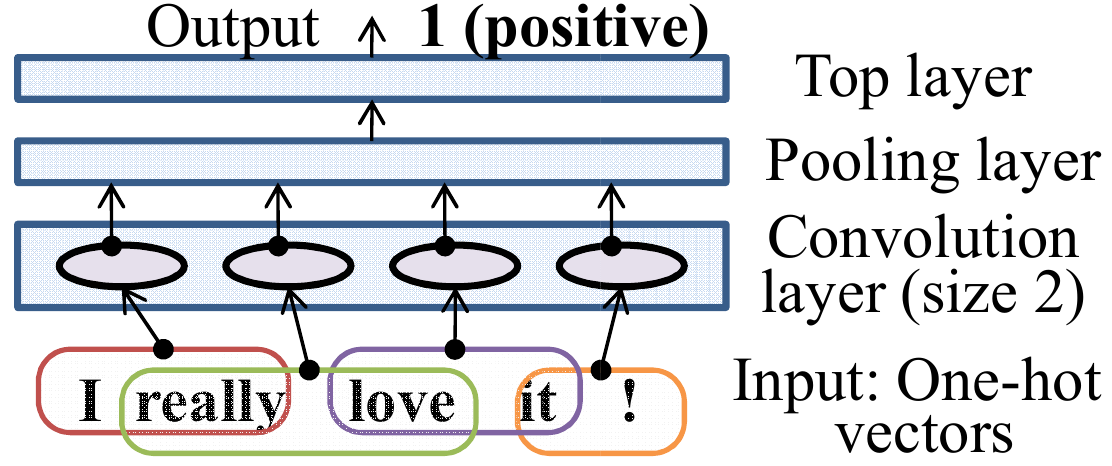}
\vspace{-0.3in}
\caption{\label{fig:textcnn} \small
One-hot \cnn\ example. 
Region size 2, stride 1. 
}
\end{minipage}
\hfill
\begin{minipage}[b]{0.42\linewidth}
\includegraphics[width=\linewidth]{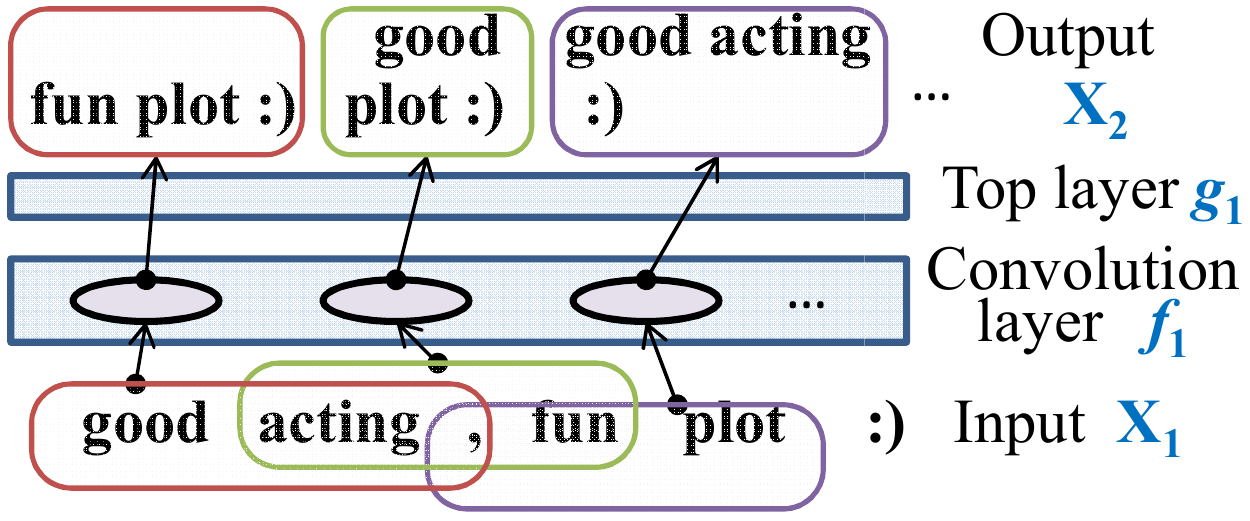}
\vspace{-0.3in}
\caption{\label{fig:unscnn} 
\small
\OurEmb\ learning by training to predict adjacent regions. 
}
\end{minipage}
\end{figure}

\section{Semi-supervised \cnn\ with \ourEmbs\ for text categorization}
\label{sec:semisup}
It was shown in \cite{JZ15} that one-hot \cnn\ is effective on text categorization, 
where the essence is direct learning of an embedding of 
text regions aided by new options of input region vector representation.  
We go further along this line and propose a semi-supervised 
learning framework that 
learns an embedding of text regions from 
unlabeled data and then integrates the learned embedding in supervised training.  
%
The first step is to learn an embedding with the following property.    
\begin{definition}[\ourEmb]
A function $f_1$ is a \ourEmb\ of 
$\cXone$ 
w.r.t. $\cXtwo$ 
if there exists a function 
$g_1$
such that $P(\Xtwo|\Xone)= g_1(f_1(\Xone),\Xtwo)$ 
for any $(\Xone,\Xtwo) \in \cXone \times \cXtwo$. 
\label{def:embed}
\end{definition}
A {\em \ourEmb} (`tv' stands for two-view) of a {\em view} ($\Xone$), by definition, preserves 
everything required to predict another view ($\Xtwo$), and it can be trained on unlabeled data. 
The motivation of \ourEmb\ is our theoretical finding 
(formalized in the Appendix) that,
essentially, a \ourEmbd\ feature vector 
$f_1(\Xone)$ is as useful as $\Xone$ for the purpose of classification {\em under ideal conditions}.  
The conditions essentially state that  
there exists a set $H$ of hidden concepts such that 
two views and labels of the classification task are related to each other 
{\em only} through the concepts in $H$.  
The concepts in $H$ might be, for example, ``pricey'', ``handy'', ``hard to use'', and so on for 
sentiment classification of product reviews.  
While in reality the ideal conditions may not be completely met, 
we consider them as guidance and design \ourEmb\ learning accordingly. 

\OurEmb\ learning is related to 
two-view feature learning \cite{AZ07} and ASO \cite{AZ05jmlr}, 
which learn a linear embedding 
from unlabeled data through tasks such as predicting a word 
(or predicted labels) from the features associated 
with its surrounding words.  
These studies were, however, limited to a linear embedding.  
A related method in \cite{DFU11} learns a word embedding so that left context and 
right context maximally correlate in terms of canonical correlation analysis. 
While we share with these studies the general idea of using the relations of two views, 
we focus on nonlinear learning of region embeddings useful for the task of interest, and 
the resulting methods are very different.  
%
An important difference of \ourEmb\ learning from co-training is that it does not involve label guessing, 
thus avoiding risk of label contamination.  
\cite{sentICML11} used a Stacked Denoising Auto-encoder to extract 
features invariant across domains for sentiment classification from unlabeled data.  
It is for fully-connected neural networks, 
which underperformed \cnns\ in \cite{JZ15}.  

\newcommand{\cSup}{{\mathcal B}}
\newcommand{\cUns}{{\mathcal U}}
Now let $\cSup$ be the base \cnn\ model  for the task of interest, and 
assume that $\cSup$ has one
convolution layer with region size $\psz$.
Note, however, that the restriction of having only one convolution layer is merely for simplifying the description.
We propose a semi-supervised framework with the following two steps.    
\vspace{-0.075in}
\begin{enumerate} 
\item
{\em \OurEmb\ learning}: 
Train a neural network 
$\cUns$ to predict the context from each region of size $\psz$ 
so that $\cUns$'s convolution layer generates feature vectors 
for each text region of size $\psz$ for use in the classifier in the top layer.  
It is this convolution layer, which embodies the \ourEmb, that we transfer to the supervised 
learning model in the next step.  
(Note that 
$\cUns$ differs from \cnn\ in that each small region is associated with its own target/output.)
\item
{\em Final supervised learning}: 
Integrate the learned \ourEmb\ (the convolution layer of $\cUns$) into $\cSup$, 
so that the \ourEmbd\ regions (the output of $\cUns$'s convolution layer) are used 
as an additional input to $\cSup$'s convolution layer.  
Train this final model with labeled data. 
\end{enumerate}
\vspace{-0.075in}
These two steps are described in more detail in the next two sections. 

\subsection{Learning \ourEmbs\ from unlabeled data} 
\label{sec:unsup}

\newcommand{\unsvec}{{\mathbf u}}
\newcommand{\ActivU}{\Activ^{(\cUns)}}
\newcommand{\WeiU}{\Wei^{(\cUns)}}
\newcommand{\weiU}{\wei^{(\cUns)}}
\newcommand{\BiasU}{\Bias^{(\cUns)}}
\newcommand{\regionU}{\region^{(\cUns)}}

We create a task on unlabeled data to predict the context (adjacent text regions) from                         
each region of size $\psz$ defined in $\cSup$'s convolution layer.  
To see the correspondence to the definition of \ourEmbs, 
it helps to consider a sub-task that assigns a label (e.g., positive/negative) 
to each text region (e.g., ``, fun plot'') instead of the ultimate task of categorizing the entire document.  
This is sensible because \cnn\ makes predictions by building up from these small regions.
In a document ``good acting, fun plot :)'' as in Figure \ref{fig:unscnn}, 
the clues for predicting a label of ``, fun plot''
are ``, fun plot'' itself (view-1: $\Xone$) and its {\em context} ``good acting'' and ``:)'' (view-2: $\Xtwo$).  
$\cUns$ is trained to predict $\Xtwo$ from $\Xone$, i.e., to approximate $P(\Xtwo|\Xone)$ 
by $g_1(f_1(\Xone),\Xtwo))$ as in Definition \ref{def:embed}, 
and functions $f_1$ and $g_1$ are embodied by the convolution layer and the top layer, respectively.  

Given a document $\bx$, for each text region indexed by $\ell$, $\cUns$'s convolution layer computes: 
\tightDisplayBegin{2pt}
\begin{align}
  \unsvec_\iL(\bx) = \ActivU \left( \WeiU \cdot \regionU_\iL(\bx) + \BiasU \right), 
  \label{eq:unsvec}
\end{align}
\tightDisplayEnd
which is the same as (\ref{eq:activ}) except for the 
superscript ``$(\cUns)$'' to indicate that these entities belong to $\cUns$.  
The top layer (a linear model for classification)  
uses $\unsvec_\iL(\bx)$ as features for prediction.  
$\WeiU$ and $\BiasU$ (and the top-layer parameters) are learned through training.   
The input region vector representation $\regionU_\iL(\bx)$ can be either sequential, \bow, 
or \bongram, independent of $\region_\iL(\bx)$ in $\cSup$. 

The goal here is to learn an embedding of text regions ($\Xone$), 
shared with all the text regions at every location.  
Context ($\Xtwo$) is used only in \ourEmb\ learning as prediction target (i.e., not transferred to the final model); 
thus, the representation of context should be determined to optimize the final outcome 
without worrying about the cost at prediction time. 
Our guidance is the conditions on the relationships between the two views mentioned above; 
ideally, the two views should be related to each other only through the relevant concepts. 
We consider the following two types of target/context representation. 
\tightpara{Unsupervised target}
A straightforward vector encoding of context/target $\Xtwo$ is 
\bow\ vectors of the text regions on the left and right to $\Xone$.  If we distinguish 
the left and right, the target vector is $2\vsz$-dimensional with vocabulary $\voc$, and if not, $\vsz$-dimensional.  
One potential problem of this encoding is that adjacent regions often have syntactic relations
(e.g., ``the'' is often followed by an adjective or a noun), 
which are typically irrelevant to the task (e.g., to identify positive/negative sentiment) and therefore undesirable.  
A simple remedy we found effective is {\em vocabulary control} of context  
to remove function words 
(or stop-words if available) 
from (and only from) the target vocabulary.  

\tightpara{Partially-supervised target} 
Another context representation that we consider is partially supervised in the sense that it uses  
labeled data.  First, we train a \cnn\ with the labeled data for the intended task and apply it 
to the unlabeled data.  Then we discard the predictions and 
only retain the internal output of the convolution layer, which is an $m$-dimensional vector 
for each text region where $m$ is the number of neurons.  We use these $m$-dimensional vectors 
to represent the context.  
\cite{JZ15} has shown, by examples, that each dimension of these vectors roughly 
represents concepts relevant to the task, e.g., `desire to recommend the product', 
`report of a faulty product', and so on. 
Therefore, 
an advantage of this representation is that there is no obvious noise 
between $\Xone$ and $\Xtwo$ since context $\Xtwo$ is represented only by the concepts 
relevant to the task.  
A disadvantage is that it is only as good as the supervised \cnn\ that produced 
it, which is not perfect and in particular, some relevant concepts would be missed if they did not 
appear in the labeled data.  
%

\subsection{Final supervised learning: integration of \ourEmbs\ into supervised \cnn}
\label{sec:finalsup}

We use the \ourEmb\ obtained from unlabeled data   
to produce {\em additional input} to $\cSup$'s convolution layer, 
by replacing  
$\Activ \left( \Wei \cdot \region_\iL(\bx) + \Bias \right)$  
(\ref{eq:activ}) 
with: 
\begin{align}
&\Activ \left( \Wei \cdot \region_\iL(\bx) + \Vei \cdot \unsvec_\iL(\bx) + \Bias \right)~, 
\label{eq:semi}
\end{align}
where $\unsvec_\iL(\bx)$ is defined by (\ref{eq:unsvec}), i.e., 
$\unsvec_\iL(\bx)$ is the output of the \ourEmb\ applied to the $\iL$-th region. 
We train this model with the labeled data of the task; that is, 
we update the weights $\Wei$, $\Vei$, bias $\Bias$, and the top-layer parameters so that 
the designated loss function is minimized on the labeled training data.  
$\WeiU$ and $\BiasU$ can be either fixed or updated for fine-tuning, 
and in this work we fix them for simplicity. 

Note that while (\ref{eq:semi}) takes a \ourEmbd\ region as input, 
(\ref{eq:semi}) itself is also an embedding of text regions; let us call it 
(and also (\ref{eq:activ}))
a {\em supervised embedding}, as it is trained with labeled data, 
to distinguish it from \ourEmbs.  
That is, we use \ourEmbs\ to improve the supervised embedding.  
Note that (\ref{eq:semi}) can be naturally extended to accommodate multiple \ourEmbs\ by 
\begin{align}
&\Activ \left(\Wei \cdot \region_\iL(\bx) + \sum_{i=1}^k{\Vei}^{(i)} \cdot \unsvec^{(i)}_\iL(\bx) + \Bias \right)~, 
\label{eq:semimore} 
\end{align}
so that, for example, two types of \ourEmb\ (i.e., $k=2$) obtained with the unsupervised target
and the partially-supervised target can be used at once, 
which can lead to performance improvement as they complement each other, 
as shown later.  

\section{Experiments}
\label{sec:semiexp}


Our code and the experimental settings are available at {\tt \small riejohnson.com/cnn\_download.html}. 

\tightpara{Data}
We used the three datasets used in \cite{JZ15}: IMDB, Elec, and RCV1, 
as summarized in Table \ref{tab:data}. 
%
IMDB 
(movie reviews) \cite{MDPHNP11} comes with an unlabeled set. 
To facilitate comparison with previous studies, 
we used a union of this set and the training set as unlabeled data. 
Elec
consists of Amazon reviews of electronics products.  
To use as unlabeled data, we chose 200K reviews from the same data source 
so that they are disjoint from the training and test sets, 
and that the reviewed products are disjoint from the test set. 
%
On the 55-way classification of the second-level topics on RCV1 (news), 
unlabeled data was chosen to be 
disjoint from the training and test sets. 
On the multi-label categorization of 103 topics on RCV1, 
since the official LYRL04 split for this task divides the entire corpus into a training set and a test set, 
we used the entire test set 
as unlabeled data (the transductive learning setting). 

\begin{table}[htb]
\begin{center}
\begin{small}
\begin{tabular}{|c|r|r|c|c|c|} 
\hline                         
              & \#train & \#test & \#unlabeled & \#class & output \\
\hline
IMDB          & 25,000  & 25,000 &  75K  (20M words)  &  2 & \multirow{1}{*}{Positive/negative} \\
\cline{1-5}
Elec          & 25,000  & 25,000 &  200K (24M words)  &  2 & sentiment \\
\hline
\multirow{2}{*}{RCV1}
              & 15,564  & 49,838 & 669K  (183M words) & 55 (single) & \multirow{2}{*}{Topic(s)} \\ 
  \cline{2-5}              
              & 23,149  &781,265 & 781K (214M words)  & 103 (multi)$\dagger$ &        \\
\hline
\end{tabular}
\end{small}
\vspace{-0.1in}
\caption{
\label{tab:data} \small
Datasets.  $\dagger$The multi-label RCV1 is used only in Table \ref{tab:rcv-semi}. 
}
\end{center}
\end{table}

\newcommand{\lossWei}{\alpha}
\tightpara{Implementation}
We used the one-layer \cnn\ models found to be effective in \cite{JZ15}  
as our base models $\cSup$, 
namely, \scnn\ on IMDB/Elec and \bcnn\ on RCV1.  
%
\OurEmb\ training 
minimized weighted square loss 
$\sum_{i,j} \lossWei_{i,j}(\bz_i[j]-\bp_i[j])^2$ 
where $i$ goes through the regions, 
$\bz$ 
represents 
the target regions, 
and $\bp$ is the model output. 
The weights $\lossWei_{i,j}$ were set to balance the loss originating from 
the presence and absence of words (or concepts in case of the partially-supervised target) 
and to speed up training by eliminating some negative examples, 
similar to negative sampling of \cite{wvecNips13}.   
%
To experiment with the unsupervised target, we set $\bz$ to be \bow\ vectors of adjacent regions 
on the left and right, while only retaining the 30K most frequent words with {\em vocabulary control}; 
on sentiment classification, 
function words were removed, 
and on topic classification, numbers and stop-words provided by \cite{LYRL04} were removed. 
Note that these words were removed 
from (and only from) the target vocabulary.  
To produce the partially-supervised target, we first trained the supervised \cnn\ models 
with 1000 neurons 
and applied the trained convolution layer to unlabeled data to generate 1000-dimensional 
vectors for each region.  
%
The rest of implementation follows \cite{JZ15}; i.e., supervised models minimized square loss with $L_2$ regularization 
and optional dropout \cite{dropout12}; $\Activ$ and $\ActivU$ were the rectifier; 
response normalization was performed; 
optimization was done by SGD.  

\tightpara{Model selection} 
On all the tested methods, 
tuning of meta-parameters 
was done by testing the models on the held-out portion 
of the training data, and then the models were re-trained with the chosen 
meta-parameters using the entire training data.  

\subsection{Performance results} 

\paragraph{Overview} 
After confirming the effectiveness of our new models in comparison with the supervised \cnn, 
we report the performances of \cite{Kim14}'s \cnn, 
which relies on word vectors pre-trained with a very large corpus (Table \ref{tab:semisup}). 
Besides comparing the performance of approaches as a whole, it is also of interest 
to compare the usefulness of what was learned from unlabeled data; 
therefore, we show how it performs if we integrate 
the word vectors into our base model one-hot \cnns\ (Figure \ref{tab:gn}). 
In these experiments we also test word vectors trained by word2vec \cite{wvecNips13} 
on our unlabeled data
(Figure \ref{fig:w2v}). 
We then compare our models with two standard semi-supervised methods, transductive SVM (TSVM) \cite{T99} and 
co-training 
(Table \ref{tab:semisup}), 
and with the previous best results in the literature (Tables \ref{tab:imdb}--\ref{tab:rcv-semi}).  
In all comparisons, our models outperform the others. 
In particular, our region \ourEmbs\ are shown to be more compact and effective than region 
embeddings obtained by simple manipulation of word embeddings, which supports 
our approach of using region embedding instead of word embedding.

\begin{table}[h]
\begin{center}
\begin{small}
\begin{tabular}{|l|l|l|} 
\hline
names in Table \ref{tab:semisup}
             & $\Xone$: $\regionU_{\iL}(\bx)$ & $\Xtwo$: target of $\cUns$ training \\
\hline             
\ourEmbUns\  & \bow\ vector           & \bow\ vector \\
\ourEmbPar\  & \bow\ vector           & output of supervised embedding \\
\ourEmbUnsN\ & bag-of-\{1,2,3\}-gram vector& \bow\ vector \\
\hline
\end{tabular}
\end{small}
\vspace{-0.1in}
\caption{\small \label{tab:ourembs}
  Tested \ourEmbs.  
}
\end{center}
\begin{center}
\begin{tabular}{c|c|c|c|r|r|r|} 
\cline{2-7}         
    & \multicolumn{3}{|c|}{}&IMDB & Elec & RCV1 \\
\cline{2-7}
1&  \multicolumn{3}{|c|}{linear SVM with 1-3grams \cite{JZ15}} &10.14& 9.16 & 10.68 \\
\cline{2-7}
2&  \multicolumn{3}{|c|}{linear TSVM with 1-3grams}            & 9.99 & 16.41 & 10.77 \\ 
\cline{2-7}
3& \multicolumn{3}{|c|}{\cite{Kim14}'s \cnn\          } & 9.17 & 8.03 &10.44\\   
\cline{2-7}
4&  \multicolumn{3}{|c|}{One-hot \cnn\ (simple) \cite{JZ15}}& 8.39 & 7.64 & 9.17 \\                                                
\cline{2-7}
5&  \multicolumn{3}{|c|}{One-hot \cnn\ (simple) co-training best} &(8.06)&(7.63)&(8.73)\\ 
\cline{2-7}
6&   &\multirow{2}{*}{\ourEmbUns} &100-dim & 7.12 & 6.96 & 8.10  \\ 
7&   &                            &200-dim & 6.81 & 6.69 & 7.97 \\ 
   \cline{3-7}  
8& \multirow{3}{*}{Our \cnn\   }   
   &\multirow{2}{*}{\ourEmbPar} &100-dim & 7.12 & 6.58 & 8.19  \\  
9&   &                            &200-dim & 7.13 & 6.57 & 7.99 \\ 
   \cline{3-7}  
10&  &\multirow{2}{*}{\ourEmbUnsN} &100-dim & 7.05 & 6.66 & 8.13  \\ 
11&  &                             &200-dim & 6.96 & 6.84 & 8.02 \\ 
   \cline{3-7}  
12&  & all three &100$\times$3&{\bf 6.51}&{\bf 6.27}&{\bf 7.71}\\
\cline{2-7}
\end{tabular}
\vspace{-0.1in}
\caption{\small \label{tab:semisup}
  Error rates (\%). 
  For comparison, 
  all the \cnn\ models were constrained to have 
  1000 neurons.  
  The parentheses around the error rates indicate that co-training meta-parameters were tuned on  
  test data.
}
\end{center}
\end{table}

\tightpara{Our \cnn\ with \ourEmbs} 
We tested three types of \ourEmb\ as summarized in Table \ref{tab:ourembs}.  
The first thing to note is that 
all of our \cnns\ (Table \ref{tab:semisup}, row 6--12) outperform 
their supervised counterpart in row 4.  
This confirms the effectiveness of the framework we propose.  
In Table \ref{tab:semisup}, for meaningful comparison, 
all the \cnns\ are constrained to have exactly one convolution layer (except for \cite{Kim14}'s \cnn)
with 1000 neurons.  
The best-performing supervised \cnns\ within these constraints (row 4) are: 
\scnn\ (region size 3) on IMDB and Elec 
and \bcnn\ (region size 20) on RCV1\footnote{
  The error rate on RCV1 in row 4 slightly differs from \cite{JZ15} because here we did not use the stopword list.
}. 
They also served as our base models $\cSup$ (with region size parameterized on IMDB/Elec).  
More complex supervised \cnns\ from \cite{JZ15} will be reviewed 
later. 
On sentiment classification (IMDB and Elec), 
the region size chosen by model selection for our models 
was 5, larger than 3 for the supervised \cnn.  
This indicates that unlabeled data enabled effective use of larger regions which are more predictive 
but might suffer from data sparsity in supervised settings.  

`\ourEmbUnsN' 
(rows 10--11) 
uses a bag-of-$n$-gram vector to initially represent each region, thus, retains word order 
partially within the region.  
When used individually, \ourEmbUnsN\ did not outperform the other \ourEmbs, which use \bow\ instead (rows 6--9).  
But we found that it contributed to error reduction when combined with the others (not shown in the table).  
This implies that it learned from unlabeled data predictive information that the other two embeddings missed.  
The best performances (row 12) were obtained by using all the three types of 
\ourEmbs\ at once according to (\ref{eq:semimore}).  
By doing so, the error rates were improved by nearly 1.9\% (IMDB) and 1.4\% (Elec and RCV1) 
compared with the supervised \cnn\ (row 4), 
as a result of the three \ourEmbs\ with different strengths 
complementing each other. 

\tightpara{\cite{Kim14}'s \cnn}
It was shown in \cite{Kim14} that \cnn\ that uses the Google News word vectors as input 
is competitive on a number of sentence classification tasks.  
These vectors (300-dimensional) were trained by the authors of word2vec \cite{wvecNips13}
on a very large Google News ({\em GN}) corpus (100 billion words; 500--5K times larger than our unlabeled data). 
\cite{Kim14} argued that these vectors can be useful for various tasks, serving as `universal feature extractors'.  
We tested \cite{Kim14}'s \cnn, which is equipped with three convolution layers with different region sizes (3, 4, and 5) and max-pooling, 
using the GN vectors as input.  
Although \cite{Kim14} used only 100 neurons for each layer, we changed it to 400, 300, and 300 to match 
the other models, which use 1000 neurons.  
Our models clearly outperform these models (Table \ref{tab:semisup}, row 3) with relatively large differences.  
\begin{figure}[t]
\centering
\begin{minipage}{0.25\linewidth}
\begin{small}
\begin{tabular}{|c|c|c|} 
\hline      
       & concat & avg \\
\hline       
IMDB & 8.31 & 7.83 \\
Elec & 7.37 & 7.24 \\
RCV1 & 8.70 & 8.62 \\
\hline
\end{tabular}
\end{small}
\vspace{-0.1in}
\caption{\small \label{tab:gn}
  GN word vectors integrated into our base models. 
  Better than \cite{Kim14}'s \cnn\ (Table \ref{tab:semisup}, row 3). 
}
\end{minipage}
\hfill
\begin{minipage}{0.68\linewidth}
\centering
\includegraphics[width=1\linewidth]{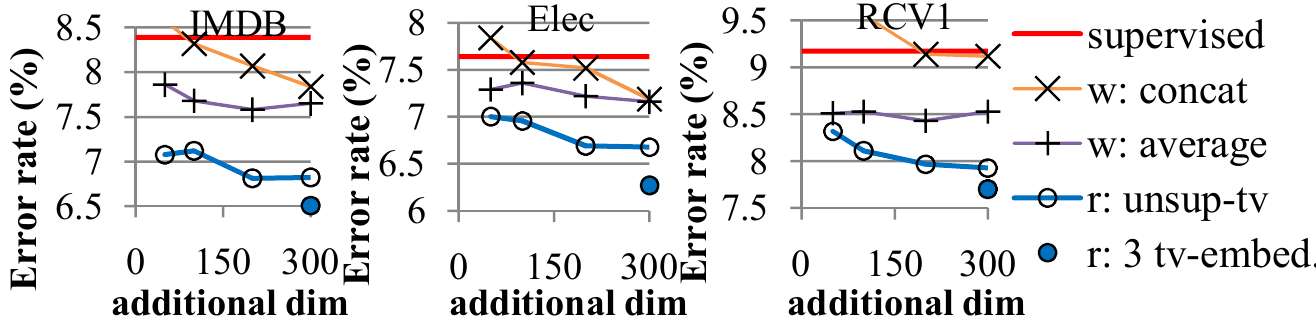}
\vspace{-0.3in}
\caption{\small \label{fig:w2v} 
Region \ourEmbs\ vs. word2vec word embeddings.  Trained on our unlabeled data. 
$x$-axis: dimensionality of the additional input to supervised region embedding. 
`r:': region, `w:': word. 
}
\end{minipage}
\end{figure}
\tightpara{Comparison of embeddings}
Besides comparing the performance of the approaches as a whole, it is also of interest 
to compare the usefulness of what was learned from unlabeled data.  
For this purpose, we experimented with integration of a word embedding into our base models 
using two methods; 
one takes the concatenation, and the other takes the average, 
of word vectors for the words in the region. 
These provide additional input to the supervised 
embedding of regions in place of $\unsvec_{\iL}(\bx)$ in (\ref{eq:semi}).  
That is, for comparison, we produce a region embedding from a word embedding to replace a region \ourEmb.  
We show the results with two types of word embeddings: the GN word embedding above (Figure \ref{tab:gn}), and 
word embeddings that we trained with the word2vec software on our unlabeled data, i.e., the same data 
as used for \ourEmb\ learning and all others (Figure \ref{fig:w2v}). 
Note that 
Figure \ref{fig:w2v} plots error rates in relation to the dimensionality of the produced 
additional input; a smaller dimensionality has an advantage of faster training/prediction.  

On the results, 
first, the region \ourEmb\ is more useful for these tasks than the tested word embeddings 
since the models with a \ourEmb\ 
clearly outperform all the models with a word embedding.  
Word vector concatenations of much higher dimensionality than those shown in the figure 
still underperformed 100-dim region \ourEmb. 
Second, 
since our region \ourEmb\ takes the form of 
$\Activ(\Wei \cdot \region_{\iL}(\bx) + \Bias)$ with $\region_{\iL}(\bx)$ being a \bow\ vector, 
the columns of $\Wei$ correspond to words, and therefore, $\Wei \cdot \region_{\iL}(\bx)$ 
is the sum of $\Wei$'s columns whose corresponding words are in the $\iL$-th region.  
Based on that, one might wonder why we should not simply use the sum or average of word vectors obtained by 
an existing tool such as word2vec instead.  
The suboptimal performances of `w: average' (Figure \ref{fig:w2v}) tells us that this is a bad idea.  
We attribute it to the fact that region embeddings learn predictiveness of 
co-presence and absence of words in a region; a region embedding can be more expressive than 
averaging of word vectors.  
Thus, an {\em effective} and {\em compact} region embedding cannot be 
trivially obtained from a word embedding. 
In particular, effectiveness of the combination of three \ourEmbs\ (`r: 3 tv-embed.' in Figure \ref{fig:w2v}) stands out. 

Additionally,  
our mechanism of using information from unlabeled data is more effective 
than \cite{Kim14}'s \cnn\ since our \cnns\ with GN (Figure \ref{tab:gn}) outperform 
\cite{Kim14}'s \cnns\ with GN (Table \ref{tab:semisup}, row 3).  
This is because 
in our model, one-hot vectors (the original features) compensate for potential information loss 
in the embedding learned from unlabeled data.  
This, as well as region-vs-word embedding, is a major difference between our model and \cite{Kim14}'s model.  

\tightpara{Standard semi-supervised methods}  
Many of the standard semi-supervised methods are not applicable to \cnn\ 
as they require \bow\ vectors as input.  
We tested TSVM with bag-of-\{1,2,3\}-gram vectors using SVMlight.  
TSVM underperformed the supervised SVM\footnote{
  Note that for feasibility, 
  we only used the 30K most frequent $n$-grams in the TSVM experiments, thus, showing 
  the SVM results also with 30K vocabulary for comparison, 
  though on some datasets SVM performance can be improved by use of all the $n$-grams 
  (e.g., 5 million $n$-grams on IMDB) \cite{JZ15}. 
  This is because the computational cost of TSVM (single-core) turned out to be high, taking several days 
  even with 30K vocabulary.  
} 
on two of the three datasets (Table \ref{tab:semisup}, rows 1--2). 
Since co-training is a meta-learner, it can be used with \cnn.  
Random split of vocabulary and split into the first and last half of each document were tested. 
%
To reduce the computational burden, we 
report the best (and unrealistic) co-training performances 
obtained by optimizing the meta-parameters
including when to stop {\em on the test data}.  
Even with this unfair advantage to co-training, 
co-training (Table \ref{tab:semisup}, row 5) clearly underperformed our models.    
The results demonstrate the difficulty of effectively using unlabeled data on these tasks, 
given that the size of the labeled data is relatively large.  

%
%


\begin{table}
\begin{minipage}{0.4\linewidth}
\centering
\begin{small}
\begin{tabular}{|l|l|c|}      
\hline
\nbw\ 1-3grams      \cite{MMRB14}    & 8.13 &     --     \\
\hline
\cite{JZ15}'s best \cnn 
                                    & 7.67 &     -- \\
\hline                                    
Paragraph vectors \cite{LM14}       & 7.46 & Unlab.data\\
\hline
Ensemble of 3 models \cite{MMRB14} & 7.43 & Ens.+unlab. \\
\hline
Our best 
      &{\bf 6.51} & Unlab.data\\ 
\hline
\end{tabular}
\end{small}
\vspace{-0.15in}
\caption{\small \label{tab:imdb}
IMDB: previous error rates (\%). 
}
\end{minipage}
\hspace{0.5in}
\begin{minipage}{0.4\linewidth}
\centering
\begin{small}
\begin{tabular}{|l|l|c|}      
\hline
SVM 1-3grams   \cite{JZ15}     & 8.71  & --\\
\hline
dense NN 1-3grams \cite{JZ15}  & 8.48  & --    \\
\hline
\nbw\ 1-3grams \cite{JZ15}     & 8.11  & --   \\
\hline
\cite{JZ15}'s best \cnn        & 7.14  & --   \\
\hline
Our best 
  &{\bf 6.27} & Unlab.data\\ 
\hline
\end{tabular}
\end{small}
\vspace{-0.15in}
\caption{\small \label{tab:elec}
Elec: previous error rates (\%). 
}
\end{minipage}
\end{table}
\begin{table}
\begin{center}
\begin{small}
\begin{tabular}{|l|c|c|c|} 
\hline
models                         & micro-F & macro-F & extra resource \\
\hline
SVM \cite{LYRL04}  & 81.6     & 60.7 & -- \\
\bcnn\ \cite{JZ15} & 84.0     & 64.8 & -- \\ 
\hline
\bcnn\ w/ three \ourEmbShort\ &{\bf 85.7}&{\bf 67.1}& Unlabeled data\\ 
\hline
\end{tabular}
\end{small}
\vspace{-0.1in}
\caption{\small \label{tab:rcv-semi} 
RCV1 micro- and macro-averaged F 
on the multi-label task (103 topics) 
with the LYRL04 split. 
}
\end{center}
\end{table}

\tightpara{Comparison with the previous best results}
We compare our models 
with the previous best results on IMDB (Table \ref{tab:imdb}). 
Our best model with three \ourEmbs\ outperforms the previous best results by nearly 0.9\%. 
All of our models with a single \ourEmbShort\ (Table \ref{tab:semisup}, row 6--11) also perform better than the previous 
results. 
Since Elec is a relatively new dataset, we are not aware of any previous semi-supervised results.  
Our performance is better than \cite{JZ15}'s best supervised \cnn, 
which has a complex network architecture of three convolution-pooling 
pairs in parallel (Table \ref{tab:elec}).  
To compare with the benchmark results in \cite{LYRL04}, we tested our model on 
the multi-label task with the LYRL04 split \cite{LYRL04} on RCV1, in which more than one 
out of 103 categories can be assigned to each document.
Our model outperforms 
the best SVM of \cite{LYRL04} and the best supervised \cnn\ of \cite{JZ15} (Table \ref{tab:rcv-semi}). 

\section{Conclusion}
\label{sec:conc}
This paper proposed a new semi-supervised \cnn\ framework for text categorization 
that learns embeddings of text regions with unlabeled data and then labeled data. 
As discussed in Section \ref{sec:textcnn},
a region embedding is trained to learn the predictiveness of co-presence 
and absence of words in a region.  
In contrast, a word embedding is trained to only represent individual words in isolation.  
Thus, a region embedding can be more expressive than simple averaging of word vectors 
in spite of their seeming similarity.  
Our comparison of embeddings confirmed its advantage; our region \ourEmbs, 
which are trained specifically for the task of interest, 
are more effective than the tested word embeddings.   
Using our new models, we were able to achieve 
higher performances than the previous studies on sentiment classification and topic classification. 


\begin{appendices}

\newcommand{\bxi}{\bx^i}  
\newcommand{\bW}{{\mathbf W}}
\newcommand{\bb}{{\mathbf b}}
\newcommand{\bbi}{{\bb_i}}  
\newcommand{\bvi}{{\bv_i}}  
\newcommand{\bu}{{\mathbf u}}
\newcommand{\bA}{{\mathbf A}}
\newcommand{\bB}{{\mathbf B}}
\newcommand{\bC}{{\mathbf C}}
\newcommand{\bU}{{\mathbf U}}

\newcommand{\RegEmb}{{RETEX}}
\newcommand{\reg}{{\cal R}}

\newcommand{\Real}{\mathbb{R}}  
\newtheorem{proposition}{Proposition}

\section{Theory of \ourEmb}


\label{sec:theory}
Suppose that we observe two views
$(\Xone,\Xtwo) \in \cXone \times \cXtwo$ of the input, and a target label $Y \in \cY$ of interest, 
where $\cXone$ and $\cXtwo$ are finite discrete sets.  
\begin{assumption}
Assume that there exists a set of hidden states $\cH$ such that 
$\Xone$, $\Xtwo$, and $Y$ are conditionally independent given 
$\hh$ in $\cH$, and that 
the rank of matrix $[P(\Xone,\Xtwo)]$ is $|\cH|$.
\label{assump:indep}
\end{assumption}
\begin{theorem}
  Consider a \ourEmb\ $f_1$ of 
  $\cXone$ w.r.t. $\cXtwo$.  
  Under Assumption 1,
  there exists a function $q_1$ such that 
  $P(Y|\Xone)=q_1(f_1(\Xone),Y)$. 
  Further consider a \ourEmb\ $f_2$ of $\cXtwo$ w.r.t. $\cXone$.  
  Then, under Assumption 1, 
  there exists a function $q$ such that 
  $P(Y|\Xone,\Xtwo)=q(f_1(\Xone),f_2(\Xtwo),Y)$. 
  \label{thm:embed}
\end{theorem}
\begin{proof}
First, assume that $\cXone$ contains $d_1$ elements, and $\cXtwo$
contains $d_2$ elements, and $|\cH|=k$. 
The independence and rank condition in Assumption~\ref{assump:indep} implies the decomposition
  \[
  P(\Xtwo|\Xone) = \sum_{h \in \cH} P(\Xtwo|h) P(h|\Xone) 
  \]
  is of rank $k$ if we consider $P(\Xtwo|\Xone)$ as a $d_2 \times d_1$
  matrix (which we denote by $\bA$). Now we may also regard $P(\Xtwo|h)$ as a $d_2 \times k$
  matrix (which we denote by $\bB$), and $P(h|\Xone)$ as a $k \times
  d_1$ matrix (which we denote by $\bC$). From the matrix equation $\bA=\bB
  \bC$, we obtain $\bC= (\bB^\top \bB )^{-1} \bB^\top \bA$. Consider the $k \times
  d_2$ matrix $\bU=(\bB^\top \bB )^{-1} \bB^\top$.  
  Then we know that its elements correspond to
  a function 
  of $(h,\Xtwo) \in \cH \times \cXtwo$. 
  Therefore the relationship $\bC=\bU\bA$ implies that there exists a function $u(h,\Xtwo)$ such that 
  \[
  \forall h \in \cH: P(h|\Xone)= \sum_{\Xtwo \in \cXtwo} P(\Xtwo|\Xone) u(h,\Xtwo) .
  \]
  Using the definition of embedding in Definition~\ref{def:embed}, we obtain
  \[
  P(h|\Xone)= \sum_{\Xtwo \in \cXtwo} g_1(f_1(\Xone),\Xtwo) u(h,\Xtwo) .
  \]
  Define $t_1(a_1,h)= \sum_{\Xtwo} g_1(a_1,\Xtwo) u(h,\Xtwo)$, then for any $h \in \cH$ we
  have 
  \begin{equation}
    P(h|\Xone)= t_1(f_1(\Xone),h) . \label{eq:h-Xone}
  \end{equation}
  Similarly, there exists a function $t_2(a_2,h)$ such that for any $h \in \cH$ 
  \begin{equation}
    P(h|\Xtwo)= t_2(f_2(\Xtwo),h) . \label{eq:h-Xtwo}
  \end{equation}

  Observe that 
  \begin{align*}
  P(Y|\Xone)&=\sum_{h \in \cH}P(Y,h|\Xone) = \sum_{h \in \cH}P(h|\Xone)P(Y|h,\Xone) \\
            &=\sum_{h \in \cH}P(h|\Xone)P(Y|h)=\sum_{h \in \cH}t_1(f_1(\Xone),h)P(Y|h)  
  \end{align*}
  where the third equation has used the assumption that $Y$ is independent of $\Xone$ given $h$ and 
  the last equation has used (\ref{eq:h-Xone}).  
  By defining $q_1(a_1,Y)=\sum_{h \in \cH} t_1(a_1,h)P(Y|h)$, we obtain $P(Y|\Xone)=q_1(f_1(\Xone),Y)$, as desired. 
  
  Further 
  observe that
  \begin{align}
  P(Y|\Xone,\Xtwo)=& \sum_{h \in \cH} P(Y,h|\Xone,\Xtwo) \nonumber \\
=& \sum_{h \in \cH} P(h|\Xone,\Xtwo) P(Y|h,\Xone,\Xtwo)  \nonumber \\
  =& \sum_{h \in \cH} P(h|\Xone,\Xtwo) P(Y|h) , \label{eq:Y-Xone-Xtwo}
\end{align}
where the last equation has used the assumption that $Y$ is independent of
$\Xone$ and $\Xtwo$ given $h$.

Note that
\begin{align*}
  P(h|\Xone,\Xtwo) =& \frac{P(h,\Xone,\Xtwo)}{P(\Xone,\Xtwo)} 
= \frac{P(h,\Xone,\Xtwo)}{\sum_{h' \in \cH} P(h',\Xone,\Xtwo)} 
  \\
  =& \frac{P(h) P(\Xone|h)P(\Xtwo|h)}{\sum_{h' \in \cH} P(h')
    P(\Xone|h')P(\Xtwo|h')} \\
  =& \frac{P(h,\Xone) P(h,\Xtwo) /P(h)}{\sum_{h' \in \cH} P(h',\Xone) P(h',\Xtwo) /P(h')} \\
  =& \frac{P(h|\Xone) P(h|\Xtwo) /P(h)}{\sum_{h' \in \cH} P(h'|\Xone) P(h'|\Xtwo) /P(h')} \\
  =& \frac{t_1(f_1(\Xone),h) t_2(f_2(\Xtwo),h) /P(h)}{\sum_{h' \in \cH}
    t_1(f_1(\Xone),h') t_2(f_2(\Xtwo),h') /P(h')} ,
\end{align*}
where the third equation has used the assumption that $\Xone$ is
independent of $\Xtwo$ given $h$, and the last equation has used
\eqref{eq:h-Xone} and \eqref{eq:h-Xtwo}.
The last equation means that $P(h|\Xone,\Xtwo)$ is a function of
$(f_1(\Xone), f_2(\Xtwo), h)$. That is, there exists a function $\tilde{t}$ such that
  $P(h|\Xone,\Xtwo)= \tilde{t}(f_1(\Xone),f_2(\Xtwo),h)$. 
  From \eqref{eq:Y-Xone-Xtwo}, this implies that 
  \[
  P(Y|\Xone,\Xtwo)= \sum_{h \in \cH} \tilde{t}(f_1(\Xone),f_2(\Xtwo),h) P(Y|h) .
  \]
  Now the theorem follows by defining $q(a_1,a_2,Y)=\sum_{h \in \cH} \tilde{t}(a_1,a_2,h) P(Y|h)$. 
\end{proof}

\section{Representation Power of Region Embedding}

We provide some formal definitions and theoretical arguments to
support the effectiveness of the type of region embedding experimented with in the main text. 

A text region embedding is a function that maps a region of text (a
sequence of two or more words) into a numerical vector.  
The particular form of region embedding we consider takes 
either sequential or bow representation of the text region as input.  
More precisely, 
consider a language with vocabulary $\voc$.  
Each word $\wrd$ in the language
is taken from $\voc$, and can be represented as a $|\voc|$
dimensional vector referred to as one-hot-vector representation.
Each of the $|\voc|$  vector components represents a
vocabulary entry. The vector representation of $\wrd$ has value one for the
component corresponding to the word, and value zeros elsewhere.
%
A text region 
of size $m$ is a sequence
of $m$ words $(\wrd_1, \wrd_2,\ldots,\wrd_m)$, where each word $\wrd_i
\in \voc$.  It can be represented as a $m |\voc|$ dimensional
vector, which is a concatenation of vector representations of the
words, as in (2) in Section 1.1 of the main text. 
Here we call this representation {\em seq-representation}.
An alternative is the {\em bow-representation} as in (3) of the main text.  

Let $\reg_m$ be the
set of all possible text regions of size $m$ in the seq-representation (or alternatively, bow-representation). 
We consider 
embeddings of a text region $\bx \in \reg_m$ in the form of 
\[
(\bW\bx+\bb)_+=\max(0, \bW\bx+\bb)~. 
\]
The embedding matrix $\bW$ and bias vector $\bb$ 
are learned by training, and the training objective depends
on the task.  
In the following, this  particular form of region embedding is referred to as {\em \RegEmb} 
(Region Embedding of TEXt), and 
the vectors produced by \RegEmb\ or the results of \RegEmb\ are referred to as {\em \RegEmb\ vectors}. 

The goal of region embedding learning is to map high-level concepts (relevant to the task of interest) 
to low-dimensional vectors.  As said in the main text, this cannot be done by word embedding learning 
since a word embedding embeds individual words in isolation (i.e., word-$i$ is mapped to vector-$i$ irrespective of its context), 
which are too primitive to correspond to high-level concepts.
For example, ``easy to use'' conveys positive sentiment, but ``use'' in isolation does not.   
Through the analysis of the representation power of \RegEmb, 
we show that 
unlike word embeddings, \RegEmb\ can model high-level concepts by using co-presence and absence of 
words in the region, which is similar to the traditional use of $m$-grams but more efficient/robust.  

First we show that for any (possibly nonlinear) real-valued function $f(\cdot)$ defined on $\reg_m$,  
there exists a \RegEmb\ so that this function can be expressed 
in terms of a linear function of \RegEmb\ vectors. 
This property 
is often referred to as {\em universal approximation} in the
literature (e.g., see \url{https://en.wikipedia.org/wiki/Universal_approximation_theorem}).

\begin{proposition}
  Consider a real-valued function $f(\cdot)$ defined on $\reg_m$. There exists an embedding matrix $\bW$,
  bias vector $\bb$, and vector $\bv$ such that
  $f(\bx)= \bv^\top (\bW\bx+\bb)_+$ for all $\bx \in \reg_m$.
\label{prop:universal}
\end{proposition}
\begin{proof}
  Denote by $\bW_{i,j}$ the entry of $\bW$ corresponding to the $i$-th row and $j$-th column. 
  Assume each element in $\reg_m$ can be represented as a $d$ dimensional vector with no more than $m$ ones (and the remaining entries are zeros).
  Given a specific $\bxi \in \reg_m$, let $S_i$ be a set of indexes $j \in \{1,\ldots,d\}$ such that 
  the $j$-th component of $\bxi$ is one.  
  We create a row $\bW_{i,\cdot}$ in $\bW$ such that $\bW_{i,j}= 2I( j \in S_i)-1$ for $1 \leq j \leq d$, 
  where $I(\cdot)$ is the set indicator function.  Let $\bbi=-|S_i|+1$ where $\bbi$ denotes the $i$-th component of $\bb$.  
  It follows that $\bW_{i,\cdot}\bx+\bbi=1$ if $\bx=\bxi$, and $\bW_{i,\cdot}\bx+\bbi\leq 0$ otherwise. 
  In this manner 
  we create one row of $\bW$ per every member of $\reg_m$.  
  Let $\bvi=f(\bxi)$.
  Then it follows that $f(\bx)=\bv^\top (\bW\bx+\bb)_+$. 
\end{proof}

The proof essentially constructs
the indicator functions of all the $m$-grams (text regions of size $m$) in $\reg_m$ and maps them to the corresponding function values. 
Thus, the representation power of \RegEmb\ is at least as good as $m$-grams, 
and more powerful than the sum of word embeddings in spite of the seeming similarity in form.  
However, it is well known that the traditional $m$-gram-based approaches, 
which assign one vector dimension per $m$-gram, 
can suffer from the data sparsity problem because an $m$-gram is useful only if it is seen in the training data.

This is where \RegEmb\ can have clear advantages. 
We show below that 
it can map similar $m$-grams (similar w.r.t. the training objective) to similar lower-dimensional vectors, 
which helps learning the task of interest.  
It is also more expressive than the traditional $m$-gram-based approaches because it can map not only co-presence 
but also absence of words 
(which $m$-gram cannot express concisely) into a single dimension.  
These properties lead to robustness to data sparsity. 
 
We first introduce a definition of a {\em simple concept}.  
\begin{definition}
  Consider $\reg_m$ of the seq-representation.  
  A high level semantic concept $C \subset \reg_m$ is called simple if it can be defined as follows.
  Let $V_1,\ldots, V_m \subset \voc$ be $m$ word groups (each word group may either represent a set of 
  similar words or the absent of certain words), and $s_1,\ldots,s_m \subset \{\pm 1\}$ be signs.
  Define $C$ such that $\bx \in C$ if and only if the $i$-th word in $\bx$
  either belongs to $V_i$ (if $s_i=1$) or $\neg V_i$ (if $s_i=-1$).  
\end{definition}
The next proposition illustrates the points above by stating that \RegEmb\ has the ability to represent a simple concept
(defined above via the notion of similar words) 
by a single dimension. 
This is in contrast to the construction in the proof of Proposition~\ref{prop:universal}, where one dimension could 
represent only one $m$-gram. 
\begin{proposition}
The indicator function of any simple concept $C$ can be  embedded into one dimension using \RegEmb.
\label{prop:concept}
\end{proposition}
\begin{proof}
  Consider a text region vector $\bx \in \reg_m$ in seq-representation that contains $m$ 
  of 
  $|\voc|$-dimensional segments, where the $i$-th segment represents the $i$-th position in the text region.  
  Let the $i$-th segment of $\bw$ 
  be a vector of zeros except for those components in $V_i$ being $s_i$.  Let $b=1-\sum_{i=1}^m (s_i+1)/2$. 
  Then it is not difficult to check that $I(\bx \in C) = (\bw^\top \bx + b)_+$.   
\end{proof}

The following proposition shows that \RegEmb\ can embed concepts that are unions of simple concepts into low-dimensional vectors. 
\begin{proposition}
  If $C \subset \reg_m$ is the union of $q$ simple concepts $C_1,\ldots,C_q$, then there exists a function $f(\bx)$ that is the linear function of $q$-dimensional
  \RegEmb\ vectors so that $\bx \in C$ if and only if $f(\bx) >0$.
  \label{prop:cunion}
\end{proposition}
\begin{proof}
  Let $\bb \in \Real^q$, and let $\bW$ have $q$
  rows, so that $I(\bx\in C_i)= (\bW_{i,\cdot} \bx + \bbi)_+$
for each row $i$, as constructed in the proof of Proposition \ref{prop:concept}.  
  Let $\bv=[1,\ldots,1]^\top \in \Real^q$.  
Then $f(\bx)=\bv^\top(\bW\bx+\bb)_+$ is a function of the desired property. 
\end{proof}

Note that $q$ can be much smaller than the number of $m$-grams in concept $C$.  
Proposition \ref{prop:cunion} shows that \RegEmb\ has the ability 
to simultaneously 
make use
of word similarity (via word groups) and 
the fact that words occur in the context, 
to reduce the embedding dimension.
A word embedding can model word similarity but does not model context.
$m$-gram-based approaches can model context but cannot model word similarity --- which means a concept/context
has to be expressed with a large number of individual $m$-grams, leading to the data sparsity problem. 
Thus, the representation power of \RegEmb\ exceeds that of single-word embedding and traditional $m$-gram-based approaches.  

\end{appendices}

\bibliographystyle{plain}
\begin{small}
\bibliography{cnn-supsemi}
\end{small}


\end{document}